\documentclass[11pt]{article}
\usepackage[margin=1in]{geometry}
\usepackage[T1]{fontenc}
\usepackage{lmodern}
\usepackage{amsmath,amssymb,amsthm}
\usepackage{booktabs,array}
\usepackage{graphicx}
\usepackage{microtype}
\usepackage[numbers,sort&compress]{natbib}
\usepackage[hidelinks]{hyperref}
\usepackage{url}
\newtheorem{proposition}{Proposition}
\theoremstyle{definition}
\newtheorem{example}{Example}
\newcommand{\KL}{D_{\mathrm{KL}}}
\newcommand{\E}{\mathbb E}

\newcommand{\cost}{C_\rho}
\newcommand{\ret}{J_{0,\rho}}
\title{IncentRL: The Trade-Off Between Preference Guidance and Task Performance}
\newif\ifauthorpending
\authorpendingtrue
\author{
\textnormal{Xuening Wu}$^{1}$,
\textnormal{Yanlan Kang}$^{1}$,
\textnormal{Shenqin Yin}$^{1}$\\
\normalfont $^{1}$Fudan University\\
}

\date{}
\begin{document}
\maketitle
\begin{abstract}
Preference-based reward shaping can guide reinforcement learning, but adding preference signals to the reward may unintentionally change the task being optimized. We address this problem with IncentRL, a framework that introduces preference guidance while explicitly characterizing its effect on external-task performance. IncentRL adds a Kullback--Leibler (KL) penalty between a specified outcome distribution and a preferred distribution. For finite discounted Markov decision processes with bounded shaping costs, we derive an external-value perturbation bound, establish a sufficient strict-action-gap condition for preserving the original optimal policy, and characterize the large-weight regime through discounted cumulative preference cost. Exact examples clarify the limits of these guarantees, including tied optima and support mismatch. We study a practical implementation using a hand-designed, distance-based outcome proxy, a fixed preference distribution, and score-weighted coefficient search. On MiniGrid DoorKey-8x8, the reported three-seed mean success rate after two million training steps reaches 98\% with coefficient 0.01, compared with 90.5\% for the reported zero-coefficient baseline, while the search progressively shifts toward smaller coefficients. Together, these results provide a principled view of the central trade-off in preference-based RL: using additional guidance to improve learning without excessively distorting the original task objective. The current experiments remain descriptive and do not yet isolate KL shaping from simpler alternatives.
\end{abstract}

\section{Introduction}
External reward and desired outcomes need not be expressed at the same level. A task may provide reward only at completion while a designer can specify which intermediate outcomes are preferable. Curiosity and information-gain methods offer one response to limited reward feedback by introducing additional learning signals~\citep{pathak2017curiosity,houthooft2016vime}. A different response is to encode task preferences directly and use them to modify the reward. Such a modification can guide behavior, but it can also change what the agent is being asked to optimize. A concrete question is when local preferences support long-horizon goal-directed behavior and when they obstruct it. For example, moving away from a goal may be necessary to reach a door or take a feasible route. Here, preferences are designer-specified mathematical quantities; no human preference measurements or cognitive model are assumed.

IncentRL uses the shaped reward
\begin{equation}
 r_\beta(s,a)=r(s,a)-\beta\KL\bigl(p(\cdot\mid s,a)\Vert q(\cdot\mid s)\bigr),
 \qquad\beta\geq0,
 \label{eq:reward}
\end{equation}
where $p$ represents outcomes associated with an action and $q$ specifies preferred outcomes. This expression raises several questions before any learning algorithm is chosen. Is the KL finite? Does the penalty preserve externally optimal behavior? What is optimized when $\beta$ becomes large? Does an implemented quantity actually represent predicted outcomes, or does it encode a hand-designed state score?

We address these questions for a restricted, explicit setting. The theoretical analysis fixes the outcome and preference distributions and studies their induced bounded cost in a finite discounted Markov decision process (MDP). We give a value-loss bound, a sufficient policy-preservation condition, and a large-weight limit statement. The proofs use standard bounded-reward perturbation and optimality arguments; they are included to establish the scope of this formulation, not as new general results in MDP theory. Exact examples demonstrate why weaker assumptions are insufficient.

We also examine an available MiniGrid implementation~\citep{incentrlcode}. It constructs a two-outcome vector from the agent's distance to a target after a transition. Its adaptation routine fits a distribution over coefficient candidates using scores from separate training runs. These mechanisms can be described precisely without identifying the distance vector with a calibrated prediction model or the fitted proposal with a Bayesian posterior.

This paper combines a well-defined preference-penalty objective, a self-contained analysis with counterexamples, and a characterization of the distance proxy and coefficient-search rule. Section~\ref{sec:empirical} evaluates fixed-coefficient shaping on MiniGrid and examines the behavior of coefficient search. A separate prospective study asks whether the observed benefit persists against information-matched shaping controls. The existing comparison does not establish a KL-specific mechanism, biological validity, or tuning-free adaptation.

\section{Related Work}
\paragraph{Reward shaping.}
Potential-based shaping adds $\gamma\Phi(s')-\Phi(s)$ to reward and preserves optimal policies under the corresponding assumptions~\citep{ng1999policy}. The penalty in Equation~\eqref{eq:reward} is generally not of this form. IncentRL instead permits a change in the optimal policy when preferences conflict with external reward. Its small-weight analysis quantifies a perturbation of the objective; it does not supply the general invariance property of potential-based shaping. Whether preserving the original policy is desirable depends on whether preferences are intended as guidance or as an additional task requirement.

\paragraph{Intrinsic motivation.}
Prediction-error curiosity~\citep{pathak2017curiosity} and information-gain exploration~\citep{houthooft2016vime} construct learning signals with different meanings from a preferred-outcome penalty. In particular, a KL between predicted and preferred outcomes is not automatically information gain about unknown dynamics. A preference distribution supplied by a designer encodes additional task information. The comparisons needed to evaluate its effect should control for access to that information.

\paragraph{Policy regularization and maximum-entropy RL.}
TRPO constrains policy change using a KL-based trust region~\citep{schulman2015trust}; PPO includes a clipped surrogate formulation and a KL-penalty alternative~\citep{schulman2017proximal}. The clipped formulation does not impose a hard KL constraint. Soft actor-critic augments its objective with action-policy entropy~\citep{haarnoja2018soft}. IncentRL's KL is over outcomes, not successive policies. Its outcome-entropy term, derived below, is likewise distinct from action entropy. A shared divergence notation does not make these objectives equivalent.

\paragraph{Preferences and probabilistic control.}
Active inference uses prior preferences over outcomes in a generative model and policy evaluation~\citep{friston2017active}. It need not modify those preferences during each decision. Our distinction is the addition of a local outcome penalty to an external-reward MDP; we do not derive a full expected-free-energy objective or its epistemic terms. The analogy to cognitive motivation provides context rather than a biological prediction.

\paragraph{Adaptive search distributions.}
Updating a sampling distribution from candidate performance is an established optimization approach, exemplified by the cross-entropy method~\citep{rubinstein1999cross}. The available IncentRL coefficient search belongs broadly to this family of adaptive sampling procedures. Its particular score-weighted moment update and prescribed concentration schedule are not identified here with the standard cross-entropy update, nor are they presented as a new general search algorithm. Sampling from a Beta distribution does not by itself constitute Bayesian inference.

\section{Method}\label{sec:method}
\subsection{Setting and support conditions}
Let $\mathcal S$ and $\mathcal A(s)$ be finite, with transition kernel $P$, bounded expected external reward $r$, discount $0\leq\gamma<1$, and initial distribution $\rho$. Let $\mathcal O$ be a finite outcome space. Fix normalized distributions $p(\cdot\mid s,a)$ and $q(\cdot\mid s)$. If $p$ is called predictive, its meaning must be specified, for example as the distribution of an outcome observation produced by the next transition. The theoretical results below concern the induced cost and do not establish that an estimated $p$ is accurate.

Define, using natural logarithms,
\begin{equation}
 c(s,a)=\sum_{o\in\mathcal O}p(o\mid s,a)
 \log\frac{p(o\mid s,a)}{q(o\mid s)}.
 \label{eq:cost}
\end{equation}
We use $0\log(0/q)=0$. A positive $p(o\mid s,a)$ paired with zero $q(o\mid s)$ makes the cost infinite. A convenient sufficient condition for boundedness is
\begin{equation}
 q(o\mid s)\geq q_{\min}>0\quad\text{for all }s,o,
 \qquad 0\leq c(s,a)\leq\log(1/q_{\min})=:C_{\max}.
 \label{eq:support}
\end{equation}
The upper bound follows because $\sum_o p_o\log p_o\leq0$ and
$-\sum_o p_o\log q_o\leq\log(1/q_{\min})$. Support inclusion alone is sufficient for a finite KL on each finite outcome space; Equation~\eqref{eq:support} supplies a transparent uniform bound.

For a preference $q_0$ with zeros, one explicit modification is
\begin{equation}
 q_\varepsilon(o\mid s)=(1-\varepsilon)q_0(o\mid s)+\varepsilon u(o),
 \qquad 0<\varepsilon<1,
 \label{eq:smoothing}
\end{equation}
where $u$ has full support. The smoothing level is an additional design parameter and changes the objective. It must be reported. In continuous outcome spaces, absolute continuity alone does not ensure finite or uniformly bounded KL; the finite-space result cannot be transferred without further assumptions.

\subsection{Returns and interpretation}
For a policy $\pi$, write
\begin{align}
 \ret(\pi)&=\E_{\rho,\pi}\!\left[\sum_{t=0}^{\infty}\gamma^t r(S_t,A_t)\right],\\
 \cost(\pi)&=\E_{\rho,\pi}\!\left[\sum_{t=0}^{\infty}\gamma^t c(S_t,A_t)\right],\\
 J_{\beta,\rho}(\pi)&=\ret(\pi)-\beta\cost(\pi).
 \label{eq:objective}
\end{align}
For finite $c$, $\beta=0$ recovers the external objective. A numerical implementation of that baseline should bypass penalty computation so that an invalid penalty cannot generate $0\times\infty$ or a NaN.

The identity
\begin{equation}
 -\KL(p\Vert q)=H(p)+\E_p[\log q(O)]
 \label{eq:entropy}
\end{equation}
separates outcome entropy from preference log-probability. At a fixed expected log preference, the entropy term favors higher, not lower, outcome entropy. Maximizing this expression is not generally an objective for making outcomes predictable. For uniform $q$ over $m$ outcomes, $c=\log m-H(p)$; this special case makes the distinction explicit. Neither outcome entropy nor proximity to a desired outcome establishes information gain about an unknown environment.

The value of $\beta$ depends on the reward scale. Multiplying $r$ and $\beta$ by the same positive constant multiplies the entire objective by that constant and leaves its maximizing policies unchanged. Consequently, $\beta=1$ has no universal interpretation as a balanced regime. The preference distribution and its smoothing also determine the effective scale of the penalty.

\subsection{Use with a learning algorithm}
For fixed $\beta$, a generic implementation can compute $c(s,a)$, form $r-\beta c(s,a)$, and pass that reward to its RL update. This specifies a new reward function, not a new policy optimizer. For instance, a tabular Q-learning target would be
\[
 r_t-\beta c(s_t,a_t)+\gamma\max_{a'}Q(s_{t+1},a'),
\]
with zero bootstrap at a true terminal state. Time-limit truncation must be distinguished from task termination according to the intended MDP. This expression alone gives no convergence guarantee for a function approximator.

If $p$ or $q$ is learned during training, the reward can change over time. Updating a predictor to fit observed transitions has a different purpose from changing it merely to make its predictions resemble preferences. The latter can reduce measured KL without changing actual outcomes. The analysis below fixes these distributions and therefore does not cover this nonstationarity. A partially observed implementation also needs a state, belief, or history representation appropriate to its problem; a raw observation is not automatically Markov.

\subsection{Distance-based outcome proxy}\label{sec:implementation}
The following construction instantiates the preference penalty using goal distance and selects its coefficient through an outer-loop search. The formulas describe the public implementation~\citep{incentrlcode}; the correspondence between that code version and the experimental runs has not been independently established. Version-specific behavior and reproducibility details are collected in Appendix~\ref{app:implementation}.

For a transition to state $s'$, let $d$ be the Euclidean distance from the successor agent position to the target. A distance-based outcome proxy is
\begin{equation}
 u(d)=\frac{e^{-d}}{1+e^{-d}},\qquad
 \widetilde p_d=(u(d),1-u(d)),\qquad q=(0.999,0.001).
 \label{eq:proxy}
\end{equation}
The public implementation evaluates the following numerical approximation to the KL penalty:
\begin{equation}
 f_\eta(d)=\sum_{i=1}^2\widetilde p_{d,i}
       \log\frac{\widetilde p_{d,i}+\eta}{q_i+\eta},
 \qquad\eta=10^{-8}.
 \label{eq:implementedcost}
\end{equation}
Equation~\eqref{eq:implementedcost} records the actual numerical adjustment; it is not the exact KL in Equation~\eqref{eq:cost}. For interpretation, let $f_0$ be the exact KL with the same two normalized distributions.

The proxy is hand-designed, not trained from transition observations. For $d\geq0$, its first component is at most $1/2$, including at the goal. Consequently, it should not be described as a calibrated probability of goal attainment. Its exact penalty satisfies
\begin{equation}
 f_0(0)\approx2.76123,\qquad
 \lim_{d\to\infty}f_0(d)=\log1000\approx6.90776.
 \label{eq:proxybounds}
\end{equation}
For $q=(q_1,q_2)$, the derivative of the binary KL with respect to $u$ is
\[
 \frac{d}{du}\KL((u,1-u)\Vert q)
 =\log\frac{u q_2}{(1-u)q_1}.
\]
It is negative for $0<u\leq1/2$ when $q_1=0.999$, while $u'(d)<0$. Thus $f_0$ increases with distance. The transformation is a nonlinear distance penalty with a nonzero constant component, not an independently learned distributional signal. No superiority over a simpler distance penalty follows from writing it as a KL.

Both agent coordinates and target information should be available to information-matched baselines.

Since the penalty is computed after the action, its natural form is
$r(s,a,s')-\beta f(s')$. If this is a fixed bounded Markov reward, its expected penalty $c(s,a)=\E[f(S')\mid s,a]$ falls within Section~\ref{sec:theory}. This mapping does not turn $\widetilde p_d$ into an action-conditioned prediction model. In episodic tasks the nonzero component is also consequential: paying a constant penalty until termination can favor earlier termination. By contrast, the same constant paid forever in a continuing discounted task adds a policy-independent value shift.

\subsection{Coefficient selection by adaptive distribution search}\label{sec:coefficient-search}
The search procedure samples coefficient candidates from a Beta proposal, trains a fresh PPO agent for each candidate and training seed, and uses final evaluation success to update the next round's proposal. The coefficient is fixed within each such training run. The procedure is therefore an outer-loop search, not a within-agent online coefficient update.

Let $i$ index candidate--seed pairs with positive score $y_i$ in a round. The procedure computes
\begin{equation}
 w_i=\frac{y_i}{\sum_jy_j},\quad
 m=\sum_iw_i\beta_i,\quad
 v=\max\left\{\sum_iw_i(\beta_i-m)^2,10^{-6}\right\}.
 \label{eq:searchmoments}
\end{equation}
It then moment-matches Beta shape parameters
\[
 t=\frac{m(1-m)}{v}-1,\qquad a_{\rm target}=mt,\quad b_{\rm target}=(1-m)t,
\]
and rescales their sum to
\begin{equation}
 \kappa_r=\min\{2(1.618)^{r+3},100\},\qquad r=0,1,\ldots.
 \label{eq:concentration}
\end{equation}
Whenever the moment match has positive shape parameters, the final distribution simplifies to
\begin{equation}
 \operatorname{Beta}\bigl(\kappa_r m,\kappa_r(1-m)\bigr).
 \label{eq:searchfinal}
\end{equation}
The empirical variance in Equation~\eqref{eq:searchmoments} no longer controls the final concentration. Valid moment matching requires $0<v<m(1-m)$ and $0<m<1$; the variance floor alone does not guarantee these conditions near a boundary. A robust implementation must check them before sampling again.

If there are no positive scores, the script leaves the first shape parameter unchanged, increments the second by ten, and rescales to the prescribed concentration. These are specified search heuristics. No likelihood for the update is defined, and a shrinking proposal is not a calibrated uncertainty statement. Even in a valid Bayesian model, posterior concentration would not by itself demonstrate better task performance or the removal of tuning choices.

\section{Theoretical Analysis}\label{sec:theory}
The results in this section hold for any fixed cost $0\leq c\leq C_{\max}$ in the finite discounted MDP above. Their applicability to a particular KL follows from establishing that bound. These are statements about optimal objectives, not rates of learning.

\begin{proposition}[External-value perturbation]\label{prop:value}
Let $\pi_0$ maximize $\ret$ and $\pi_\beta$ maximize $J_{\beta,\rho}$. Then
\begin{equation}
 0\leq\ret(\pi_0)-\ret(\pi_\beta)
 \leq\beta\bigl(\cost(\pi_0)-\cost(\pi_\beta)\bigr)
 \leq\frac{\beta C_{\max}}{1-\gamma}.
 \label{eq:valuebound}
\end{equation}
\end{proposition}
\begin{proof}
Shaped optimality gives
$\ret(\pi_\beta)-\beta\cost(\pi_\beta)
\geq\ret(\pi_0)-\beta\cost(\pi_0)$.
Rearranging proves the middle inequality. External optimality proves the lower bound, and $0\leq\cost(\pi)\leq C_{\max}/(1-\gamma)$ proves the last inequality.
\end{proof}

This result does not require a unique external optimum. If a learned policy is only $\epsilon$-optimal for the shaped objective, the same argument yields the weaker external-value bound $\beta C_{\max}/(1-\gamma)+\epsilon$. It does not supply a value of $\epsilon$ for an actual training run.

\begin{proposition}[Policy preservation under a strict action gap]\label{prop:gap}
Assume every state has a unique externally optimal action $a^*(s)$. Over the states with alternative actions, define
\begin{equation}
 \Delta=\min_{s:\,|\mathcal A(s)|>1}\min_{a\ne a^*(s)}
 \bigl[Q_0^*(s,a^*(s))-Q_0^*(s,a)\bigr]>0.
 \label{eq:gap}
\end{equation}
If $C_{\max}>0$ and $\beta C_{\max}/(1-\gamma)<\Delta$, then $s\mapsto a^*(s)$ remains optimal from every state for the shaped reward. If $C_{\max}=0$, or there is no state with alternative actions, the conclusion holds without this threshold.
\end{proposition}
\begin{proof}
Let $B=\beta C_{\max}/(1-\gamma)$. For each initial state-action pair and every continuation policy, subtracting the discounted penalty reduces return by an amount in $[0,B]$. Taking suprema over continuation policies gives
\[
 Q_0^*(s,a)-B\leq Q_\beta^*(s,a)\leq Q_0^*(s,a).
\]
For every alternative action,
\begin{align*}
 Q_\beta^*(s,a^*(s))-Q_\beta^*(s,a)
 &\geq Q_0^*(s,a^*(s))-Q_0^*(s,a)-B\\
 &\geq\Delta-B>0.
\end{align*}
Thus the maximizing action at each state is unchanged and defines an optimal stationary policy.
\end{proof}

The strict gap matters. Multiple external optima can have different preference costs and be separated by any positive penalty weight. Moreover, the collection of stochastic policies need not have a positive gap between optimal and suboptimal returns, even in a finite MDP. Proposition~\ref{prop:gap} is a sufficient condition, not a practical certificate unless the gap and cost bound are known.

\begin{proposition}[Large-weight limit]\label{prop:large}
Suppose $|r(s,a)|\leq R_{\max}$, and let $C_{\min}=\min_\pi\cost(\pi)$. For $\beta>0$, every maximizer $\pi_\beta$ of $J_{\beta,\rho}$ satisfies
\begin{equation}
 0\leq\cost(\pi_\beta)-C_{\min}
 \leq\frac{2R_{\max}}{\beta(1-\gamma)}.
 \label{eq:largebeta}
\end{equation}
Restricting the sequence to stationary randomized policies, every accumulation point as $\beta\to\infty$ minimizes $\cost$ and maximizes $\ret$ within that minimum-cost set. A unique policy limit is not asserted.
\end{proposition}
\begin{proof}
Choose a cost minimizer $\pi_C$. Optimality implies
\[
 \beta\bigl(\cost(\pi_\beta)-C_{\min}\bigr)
 \leq\ret(\pi_\beta)-\ret(\pi_C)
 \leq\frac{2R_{\max}}{1-\gamma}.
\]
This proves the bound. The space of stationary randomized policies is a finite product of probability simplices and is therefore compact. Its discounted reward and cost values are continuous. For example, these values have the form $(I-\gamma P_\pi)^{-1}r_\pi$ and $(I-\gamma P_\pi)^{-1}c_\pi$, and the inverse exists for every policy. Any convergent subsequence therefore has a limit with cost $C_{\min}$. For any minimum-cost policy $\pi$,
\[
 \ret(\pi_\beta)\geq\ret(\pi)
       +\beta\bigl(\cost(\pi_\beta)-C_{\min}\bigr)
 \geq\ret(\pi).
\]
Passing to the same subsequence proves external-return maximization within the minimum-cost set.
\end{proof}

The quantity minimized in this limit is a cumulative control cost. Its optimum depends on dynamics, discounting, the outcome model, preferences, and the initial distribution. A minimum-cost statement for one $\rho$ imposes no constraint on behavior at states that cannot affect its value. It is not a prescription to greedily minimize immediate KL at every state.

\subsection{An illustrative reward--preference trade-off}
The following exact example illustrates the policy change permitted by the bounds above. It is a deterministic calculation, not a learning experiment.

\begin{example}[An explicit reward--preference trade-off]\label{ex:tradeoff}
Consider a one-step decision, equivalently a problem with $\gamma=0$. The aligned action has reward zero and $p_a=q=(3/4,1/4)$, hence zero cost. The reward-seeking action has reward one and $p_b=(0,1)$, hence cost $\log4$. Their shaped values are zero and $1-\beta\log4$, so the maximizing action switches at
\[
 \beta_c=\frac{1}{\log4}\approx0.72135.
\]
Below the threshold the external return is one; above it the external return is zero. At the threshold both actions, and their mixtures, are optimal. Here $\Delta=1$ and $C_{\max}=\log4$, so the sufficient threshold in Proposition~\ref{prop:gap} matches the switch. The example shows why strong preferences can deliberately sacrifice external reward.
\end{example}

Additional examples in Appendix~\ref{app:examples} examine support failure, ties among externally optimal policies, and the difference between immediate and cumulative cost. Table~\ref{tab:exact} in that appendix gives numerical values for Example~\ref{ex:tradeoff}.

\section{Experiments}\label{sec:empirical}
Our experiments examine two aspects of IncentRL: the effect of a fixed shaping coefficient on MiniGrid learning, and the evolution of sampled coefficients across search rounds. The fixed-coefficient comparison measures task success and episode length, while the search analysis characterizes the coefficient distribution. Together they connect the preference-penalty formulation to observed learning behavior. The six-group mechanism study in Section~\ref{sec:prospective} is a separate, planned extension and is not part of the results presented here.

\subsection{Fixed-coefficient learning curves}
We use MiniGrid DoorKey-8x8, requiring the agent to obtain a key, unlock a door, and reach the goal. We compare coefficients $\beta=0$ and $\beta=0.01$, with the same architecture and training hyperparameters across conditions, two million training steps per agent, and three independent training seeds. Success rate and mean episode length are tracked during training. The correspondence between the experiment-generating code and the public version examined in Appendix~\ref{app:implementation} is not established by the available records. The implementation issues in Appendix~\ref{app:implementation} therefore cannot be assumed to have affected this comparison.

Table~\ref{tab:experiment-config} summarizes the experimental settings documented for these figures and the remaining reporting gaps. The public implementation provides candidate configuration details, but without a run-to-code link those details cannot establish the settings used to generate these results. In particular, the prospective protocol in Section~\ref{sec:prospective} is not used to fill these gaps.

\begin{table}[!htbp]
\centering
\small
\caption{Documented settings and reporting gaps for the MiniGrid experiments. Unavailable details are not inferred from library defaults or from the prospective protocol.}
\label{tab:experiment-config}
\begin{tabular}{@{}p{0.30\linewidth}p{0.65\linewidth}@{}}
\toprule
Item & Available information \\
\midrule
Environment & MiniGrid DoorKey-8x8 \\
Fixed coefficients & $\beta=0$ (reported baseline) and $\beta=0.01$ \\
Training budget & Two million reported steps per agent \\
Training seeds & Three for the fixed-coefficient comparison; identifiers not supplied. Search plots show seeds 42 and 43. \\
Architecture and optimizer & Reported as shared across fixed-coefficient conditions; exact network and optimizer settings are not recorded in the available experiment documentation. \\
Observations and versions & Observation encoding, wrappers, and environment/library versions are not documented for the figure-generating runs. \\
Evaluation & Success rate and mean episode length at training checkpoints; episode count, reset seeds, and action-selection rule are not documented. \\
Episode-length averaging & Whether the mean includes all episodes or only successful episodes is not specified. \\
Search summaries & Four rounds; per-round mean, SD, minimum, and maximum. Candidate counts and the cross-seed aggregation rule are not documented. \\
\bottomrule
\end{tabular}
\end{table}

\begin{figure}[!htbp]
\centering
\includegraphics[width=0.82\linewidth]{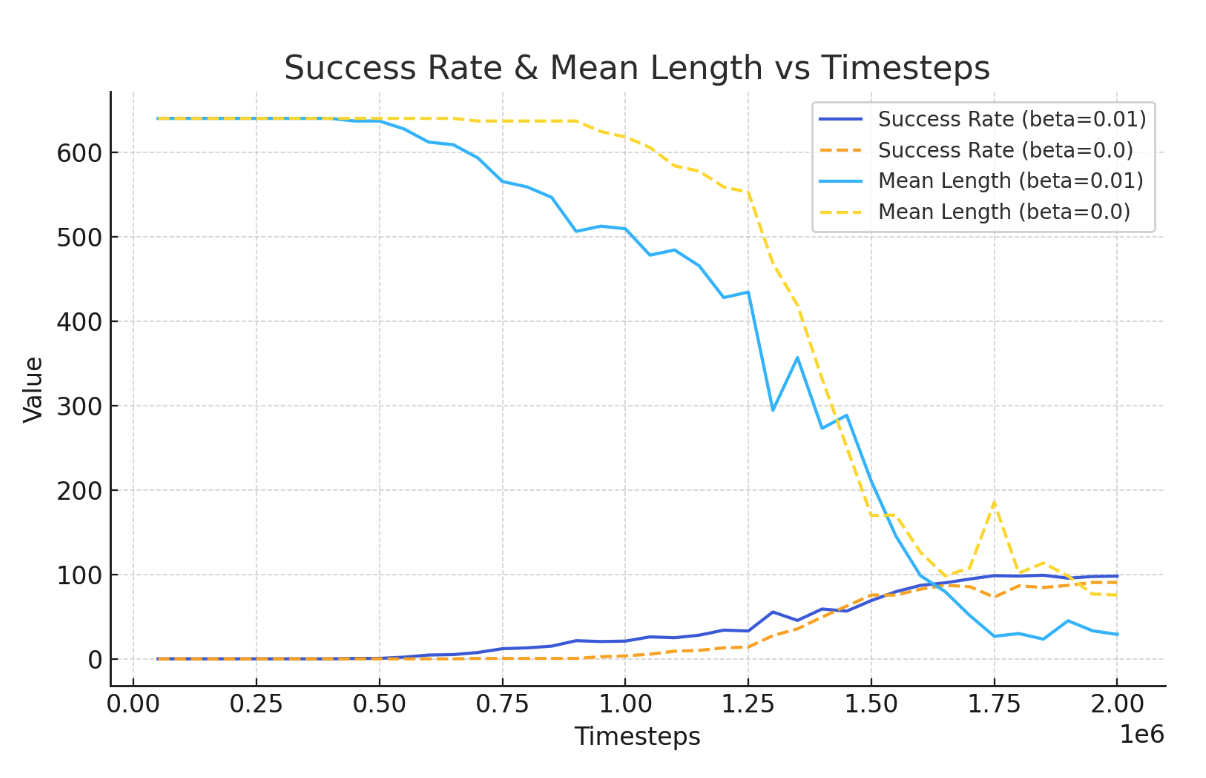}
\caption{Learning curves on MiniGrid DoorKey-8x8 for $\beta=0$ and $\beta=0.01$, averaged over three seeds with two million training steps per agent. The shared vertical axis combines success rate in percent and episode length in steps. No uncertainty bands are shown; the plot supports a descriptive comparison of the mean trajectories.}
\label{fig:minigrid}
\end{figure}

At the final reported checkpoint, the mean success rate is 98\% for $\beta=0.01$ and 90.5\% for the reported $\beta=0$ baseline, an absolute difference of 7.5 percentage points. The reported mean episode lengths are 29 and 75 steps, respectively. These endpoints and the trajectories in Figure~\ref{fig:minigrid} support a descriptive advantage in this particular comparison. They do not establish statistical significance, asymptotic superiority, or a general improvement across tasks. The available aggregate statistics do not specify whether episode length is averaged over all episodes or only successful ones, so the length difference is not interpreted as a controlled measure of path efficiency. Run-level data would be needed to quantify uncertainty and analyze learning-curve area under a defined evaluation protocol.

Both conditions improve late in training, with substantial success gains after roughly 1.25 million steps. No causal attribution to exploration or policy-update stabilization follows from these curves alone. Unshaped reward is the only displayed control: the comparison does not separate distance information, the nonlinear transformation, and a constant per-step cost. It also does not establish that coefficient search outperforms fixed-coefficient or random search at equal total cost.

\subsection{Coefficient-search observations}
Table~\ref{tab:searchsummary} summarizes sampled coefficients across search rounds, and Figure~\ref{fig:searchrounds} shows their mean and standard deviation. The mean decreases from 0.1057 in round one to 0.0173 in round four. Dispersion decreases through round three, then increases from 0.0098 to 0.0182 in round four. The spread of sampled coefficients is not a confidence interval for an optimal coefficient or calibrated posterior uncertainty.

\begin{table}[!htbp]
\centering
\caption{Summary statistics of sampled coefficients across four search rounds. Values are shown to four decimal places. The available aggregates do not specify the exact pooling or averaging rule across seeds.}
\label{tab:searchsummary}
\begin{tabular}{rrrrr}
\toprule
Round & Mean $\beta$ & SD & Minimum & Maximum \\
\midrule
1 & 0.1057 & 0.0573 & 0.0078 & 0.1805 \\
2 & 0.0454 & 0.0322 & 0.0074 & 0.0977 \\
3 & 0.0194 & 0.0098 & 0.0035 & 0.0351 \\
4 & 0.0173 & 0.0182 & 0.0000 & 0.0587 \\
\bottomrule
\end{tabular}
\end{table}

\begin{figure}[!htbp]
\centering
\includegraphics[width=0.85\linewidth]{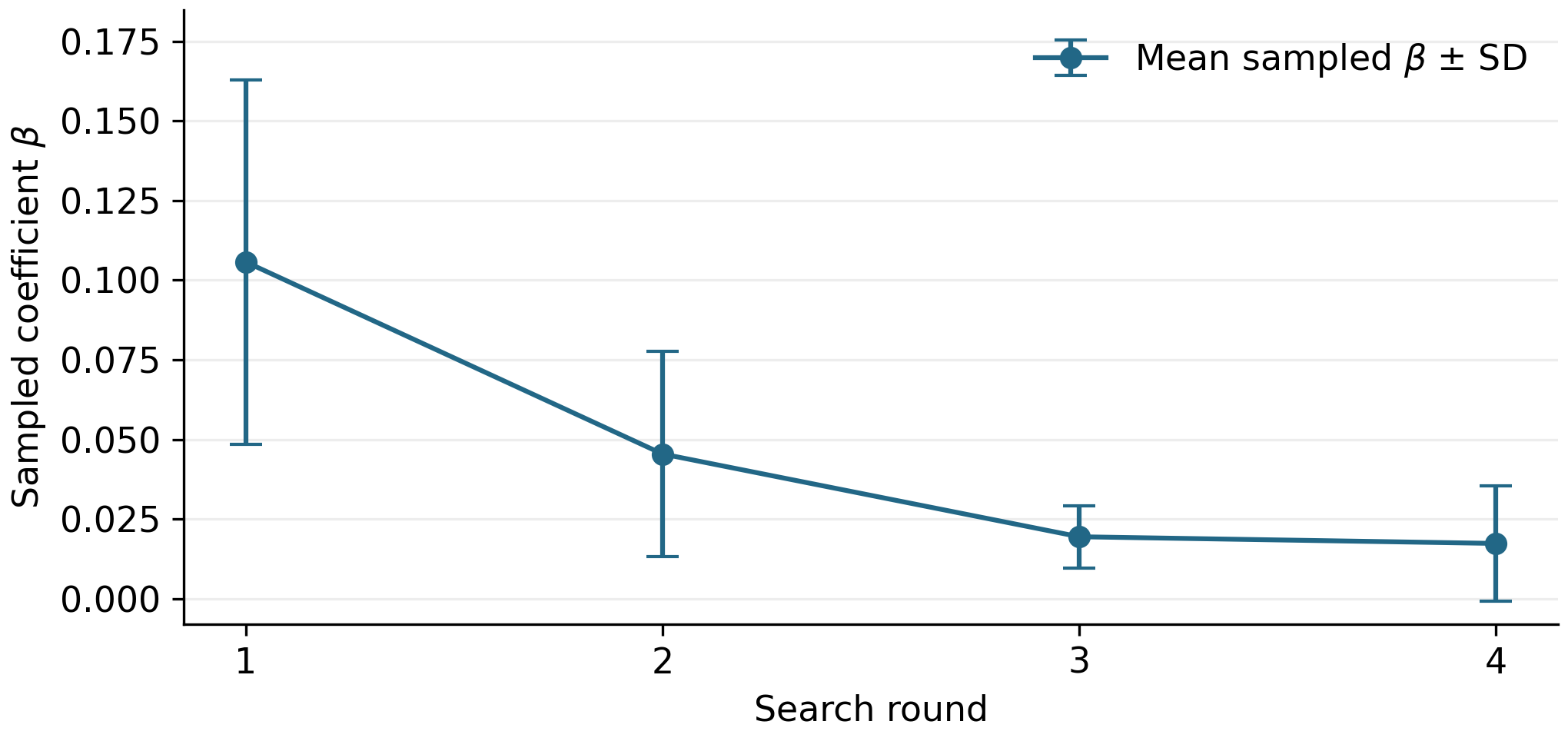}
\caption{Mean and standard deviation of sampled coefficients across search rounds, using the values in Table~\ref{tab:searchsummary}. Bars denote the reported standard deviation of sampled coefficients, not uncertainty in the mean or a Bayesian credible interval. A symmetric mean-plus/minus-SD bar can extend below zero even when all samples are nonnegative. No raw samples are reconstructed.}
\label{fig:searchrounds}
\end{figure}

These summaries document coefficient-search behavior. Supplementary distributions for individual seeds are provided in Appendix~\ref{app:search-distributions} (Figure~\ref{fig:searchsamples}). Interpreting them as an outer-loop adaptive search is consistent with the available procedure in Section~\ref{sec:coefficient-search}; a Bayesian interpretation would require a specified likelihood and posterior update. A numerical resemblance between later sampled values and a useful fixed coefficient is not a held-out performance comparison or a demonstration that tuning has been eliminated.

\section{Discussion and Limitations}
The evidence combines mathematical analysis, exact examples, implementation characterization, and MiniGrid experiments. The empirical comparison is descriptive: the available records do not establish exact experiment-generating versions, per-seed evaluation results, or all aggregation details. These gaps limit reproducibility and uncertainty quantification. Appendix~\ref{app:scope} summarizes the scope of the claims.

The theory assumes a fixed bounded penalty, exact optimization when referring to $\pi_\beta$, and a finite discounted MDP. It does not establish convergence for neural function approximation, learned or changing preferences, or adaptive $\beta_t$. The worst-case bound can be loose, especially when $\gamma$ is close to one or preference probabilities are very small. The strict action gap may also be unknown or absent. These limitations constrain both theoretical interpretation and practical parameter selection.

Preference quality is an independent issue. A task-informed penalty can be useful guidance or a source of bias. Euclidean distance can conflict with necessary detours, unlocking a door, or building momentum. A preference penalty does not guarantee exploration, safety, or transfer to a new task. Learned outcome models introduce additional approximation and calibration questions that the distance proxy does not answer.

\subsection{A prospective test of local guidance and task completion}\label{sec:prospective}
The next empirical question is whether a nonlinear distance penalty improves learning beyond direct distance guidance when observations, coefficient-selection opportunities, and training budgets are matched. The associated mechanism question is whether any difference depends on a constant cost paid until termination. These questions follow from the proxy analysis; they do not assume that the KL construction is beneficial. A locally frozen design, dated 18 September 2026, specifies the following study. Its implementation and training results are pending, and it has not been externally preregistered.

Let $f_0$ denote the exact proxy KL from Section~\ref{sec:implementation}, $D$ a fixed upper distance bound, $K=f_0(D)$, and $c_0=f_0(0)/K$. The planned penalties are normalized for an explicit coefficient grid, and differ from the unnormalized coefficients in Section~\ref{sec:empirical}. With $d'$ measured at the actual successor state, the six groups are:
\begin{center}
\begin{tabular}{@{}ll@{}}
\toprule
Group & Training reward \\
\midrule
Unshaped & $r_{\rm ext}$ \\
Constant cost & $r_{\rm ext}-\lambda c_0$ \\
Direct distance & $r_{\rm ext}-\lambda d'/D$ \\
Potential-based & $r_{\rm ext}+\lambda[\gamma\Phi(s')-\Phi(s)]$ \\
Full KL proxy & $r_{\rm ext}-\lambda f_0(d')/K$ \\
Centered KL proxy & $r_{\rm ext}-\lambda[f_0(d')/K-c_0]$ \\
\bottomrule
\end{tabular}
\end{center}
Here $\Phi(s)=-d/D$, with zero successor potential at every task terminal, including the finite horizon. Thus the discounted sum of the weighted shaping terms telescopes to $-\lambda\Phi(s_0)$~\citep{ng1999policy}. The other penalties apply on the final real transition but not after termination. Full and centered KL differ by exactly $-\lambda c_0$ per transition; this algebraic identity does not imply an additive decomposition of learning performance.

A small finite-horizon MDP will contrast a direct route with a condition requiring an initial detour, using exact dynamic programming to compare optimal action sets and task outcomes. This planned model is distinct from the completed exact examples in this note. The learning experiment will use fully observed symbolic states in DoorKey-8x8, a 640-step task horizon, and the same PPO architecture and goal information in all groups. This planned study introduces six groups, symbolic observations, normalized coefficients, and independent confirmation seeds. These settings differ from the two-condition experiment in Section~\ref{sec:empirical} and should not be used to describe its results. Real goal coordinates and terminal observations must be checked before training. Potential-based invariance concerns the underlying objective, not equality of finite-budget PPO learning curves.

The primary comparison is full KL versus direct distance, each selected using the same grid $\lambda\in\{0,0.003,0.01,0.03,0.1\}$ and three tuning seeds. The centered group inherits the full-KL coefficient and is diagnostic rather than an independently tuned competitor. Ten fresh training seeds will provide paired differences in normalized external-success learning-curve area at 17 fixed checkpoints, each evaluated on the same 100 held-out reset seeds. Each full run has 2,097,152 training transitions. The analysis uses a paired training-seed bootstrap with 20,000 resamples and a 95\% interval; the prespecified practical margin is 0.05 AUC. Inference is conditional on the chosen coefficients and fixed evaluation bank. Distinct reset seeds do not guarantee distinct layouts. This comparison estimates the difference between tuned learning procedures, not the pure effect of KL at a matched penalty amplitude. All selection costs and technical failures must be recorded. The protocol fixes the detailed selection, failure, and interpretation rules before confirmation results are examined.

\subsection{Extensions beyond the planned mechanism study}
The frozen study does not test adaptive coefficient search, learned outcome predictions, or generalization across task families. Evaluating the adaptive search would require grid and random-search controls at matched total environment interactions, with final evaluation separate from selection scores. A learned $p$ would require an explicit outcome horizon, training data and loss, and calibration assessment. Any additional benchmark must specify environment versions, reward transformations, termination handling, and run-level provenance. For example, standard MountainCar assigns $-1$ per step~\citep{mountaincar}; a sparse $+1/0$ version must be labeled as a reward-transformed task.

A full Bayesian extension would need to define the uncertain quantity, prior, observation model, and likelihood, and derive an update consistent with them. It would also need to distinguish posterior inference from the decision rule that chooses a coefficient. A Beta-shaped sampling distribution is insufficient. Such an extension should be evaluated against the explicit heuristic characterized here, with matching tuning budgets and held-out performance measures.

\section{Conclusion}
IncentRL expresses a preference-guided modification of an external-reward objective through an outcome KL penalty. Making its support and boundedness assumptions explicit yields a value-loss bound, a sufficient strict-gap condition for policy preservation, and a cumulative-cost interpretation of the large-weight limit. Exact examples show why unconditional preservation and immediate-KL minimization are not valid substitutes for these statements. The available implementation uses a distance-derived proxy and an outer-loop adaptive search distribution. The prospective study isolates local distance guidance, episode-length costs, and conflicts with task completion. It supplies a route from this analysis to a mechanism-focused empirical evaluation without presupposing a performance gain. The MiniGrid experiment shows higher endpoint success for the nonzero coefficient, while the search summaries show movement toward smaller coefficients. Neither establishes a KL-specific advantage or Bayesian adaptation; those stronger claims remain separate work.

% \section*{Data and code availability}

% \section*{Ethics and research scope}
% This paper reports mathematical analysis, inspection of public software, and computational experiments. It reports no new human-participant or animal study, and makes no claim of biological validation or certified safety of an RL agent.

% \section*{Author contributions}
% \ContributionStatement
% \section*{Funding and competing interests}
% \FundingStatement\par
% \ConflictStatement
% \section*{AI assistance}
% AI assistance was used to revise prose, formulate and check mathematical arguments, inspect the public implementation, verify bibliographic information, and prepare the LaTeX and numerical-checking materials. This disclosure does not imply an independent peer review or reproduction of benchmark experiments. The submitting authors are responsible for validating and approving the final content.

\begingroup
\raggedright
\bibliographystyle{plainnat}
\bibliography{references}
\endgroup

\appendix
\section{Scope and interpretation}\label{app:scope}
The paper studies a preference-penalized objective, its mathematical properties, a distance-based implementation, and MiniGrid learning and coefficient-search behavior. The empirical results support a descriptive comparison in the evaluated setting. They do not establish an online Bayesian update, elimination of coefficient tuning, or a general advantage across tasks. The six-group study in Section~\ref{sec:prospective} is a future evaluation designed to distinguish candidate mechanisms.

The theoretical guarantees depend on their stated assumptions. When external optima are tied, an arbitrary externally optimal policy need not remain optimal at every sufficiently small positive weight. Proposition~\ref{prop:gap} gives a sufficient strict-gap condition, while Proposition~\ref{prop:value} bounds external-value loss without that condition. At large weight, the relevant objective is discounted cumulative cost. Support conditions such as those illustrated in Example~\ref{ex:support} are necessary before finite-return arguments can be applied.

The public implementation inspected here is fixed by the full commit identifier
\begin{center}\small\texttt{f512c32b835fee07832e3044280a00a2381ad8ba}.\end{center}
Static analysis of this version does not establish the exact code used for each experimental run. Reproduction requires run identifiers, generating code versions, per-seed records, and aggregation code. The supplied figures and coefficient summaries do not replace those records.

\section{Additional examples and numerical values}\label{app:examples}
These fully specified examples illustrate the assumptions and limits of Section~\ref{sec:theory}. They are analytic calculations, not RL training experiments.

\begin{example}[A preference support failure]\label{ex:support}
With predicted outcomes $p=(0.3,0.7)$ and preferences $q=(1,0)$, the second outcome has positive predicted mass and zero preferred mass. Hence $\KL(p\Vert q)=\infty$. Replacing $q$ with $(0.999,0.001)$ gives a finite value of approximately $4.22486$. This replacement defines a different objective; it cannot justify finite numerical results from an implementation whose support handling is unspecified.
\end{example}

\begin{example}[Ties need not survive shaping]\label{ex:ties}
A one-step decision has two actions with external reward one. Their outcome distributions are $p_a=(1,0)$ and $p_b=(0,1)$, and $q=(3/4,1/4)$. The corresponding costs are $\log(4/3)$ and $\log4$. Both actions are externally optimal, but action $b$ is strictly suboptimal under every $\beta>0$. This does not contradict Proposition~\ref{prop:value}; the external return can remain unchanged even though the set of optimal policies shrinks.
\end{example}

\begin{example}[Immediate mismatch can be misleading]\label{ex:greedy}
An initial action $a$ incurs cost zero and leads to an absorbing state whose cost is one at every subsequent step. Action $b$ incurs cost one and leads to a zero-cost absorbing state. All external rewards are zero. At $\gamma=0.9$, the discounted costs are nine for $a$ and one for $b$. Immediate-cost minimization chooses the wrong action for cumulative preference alignment. These costs can be realized with a fixed binary preference $q=(e^{-1},1-e^{-1})$, using $p=q$ wherever the cost is zero and $p=(1,0)$ wherever it is one. The outcome variable may be an auxiliary observation emitted on a transition; it need not identify the next state.
\end{example}

\subsection{Values for the main-text trade-off example}
Table~\ref{tab:exact} gives values obtained directly from the two action values in Example~\ref{ex:tradeoff}.
\begin{table}[!htbp]
\centering
\caption{Exact values for Example~\ref{ex:tradeoff}. No agent is trained. The external reward is sacrificed once the weight exceeds $1/\log4$.}
\label{tab:exact}
\begin{tabular}{rlrrr}
\toprule
$\beta$ & Optimal action & External return & Preference cost & Shaped return \\
\midrule
0.00 & Reward-seeking & 1 & 1.38629 & 1.00000 \\
0.25 & Reward-seeking & 1 & 1.38629 & 0.65343 \\
0.50 & Reward-seeking & 1 & 1.38629 & 0.30685 \\
0.75 & Aligned & 0 & 0.00000 & 0.00000 \\
1.00 & Aligned & 0 & 0.00000 & 0.00000 \\
\bottomrule
\end{tabular}
\end{table}

\section{Implementation and reproducibility details}\label{app:implementation}
The public MiniGrid script \texttt{KL\_minigrid\_mp\_bayes.py} is inspected at commit \texttt{f512c32}~\citep{incentrlcode}. The observations below concern that version and do not establish which code generated individual experimental runs.

\subsection{Position lookup and target coordinates}
After stepping the environment, the wrapper attempts to read the agent position and computes distance to $(W-2,H-2)$.
If the position is unavailable, the script uses $(0.5,0.5)$ instead, giving a constant penalty. The relevant wrapper-attribute behavior depends on installed library versions and requires runtime validation. The target coordinate must also be verified against the generated environment. Both agent coordinates and target information are additional task information and should be available to matched baselines.

\subsection{Baseline handling and search accounting}
Several details prevent interpreting the inspected script as a completed controlled comparison. A zero coefficient appended as a baseline is clipped to $10^{-5}$; this search baseline is therefore not exactly unshaped. The result loader uses a wildcard across rounds and selects its first matching file, allowing an earlier result to be attributed to a later round. These are identifiable code paths, not evidence that the experimental results were affected. Exact run identifiers and an independent zero-penalty baseline are needed before reproducing a comparison.

The default configuration specifies five rounds, ten candidates per round, one additional final-round baseline, three training seeds, and two million target training steps per run. This amounts nominally to
\[
 (5\cdot10+1)\cdot3\cdot2{,}000{,}000=306{,}000{,}000
\]
training steps, excluding evaluation and rollout-size rounding. This is configuration arithmetic, not a measured execution total. Comparing the search procedure against a single two-million-step fixed-coefficient run would omit the cost of selection. Held-out final evaluation must also be separated from the scores used to fit the proposal.

\section{Details of the exact calculations}\label{app:verification}
The exact-example table uses coefficient values $0$, $0.25$, $0.50$, $0.75$, and $1$. For each coefficient, the optimal action is determined by comparing the shaped values $0$ and $1-\beta\log4$. Selecting the aligned action gives external return and preference cost both equal to zero; selecting the reward-seeking action gives external return $1$ and preference cost $\log4$. These are deterministic calculations, not estimates from training runs; no training-seed uncertainty or confidence interval is associated with them. In the support-failure example, the KL is infinite because an outcome with positive probability under $p$ has zero probability under $q$.

\section{Supplementary coefficient-search distributions}\label{app:search-distributions}
Figure~\ref{fig:searchsamples} shows the sampled-coefficient distributions for seeds 42 and 43 across four search rounds, supplementing the aggregate summaries in Section~\ref{sec:empirical}. The two seeds displayed here differ in number from the three seeds used for the fixed-coefficient learning curves. These distributions illustrate variation between seeds and a shift toward smaller sampled coefficients; they do not establish posterior inference or a performance advantage from coefficient search.

\begin{figure}[!htbp]
\centering
\includegraphics[width=0.90\linewidth]{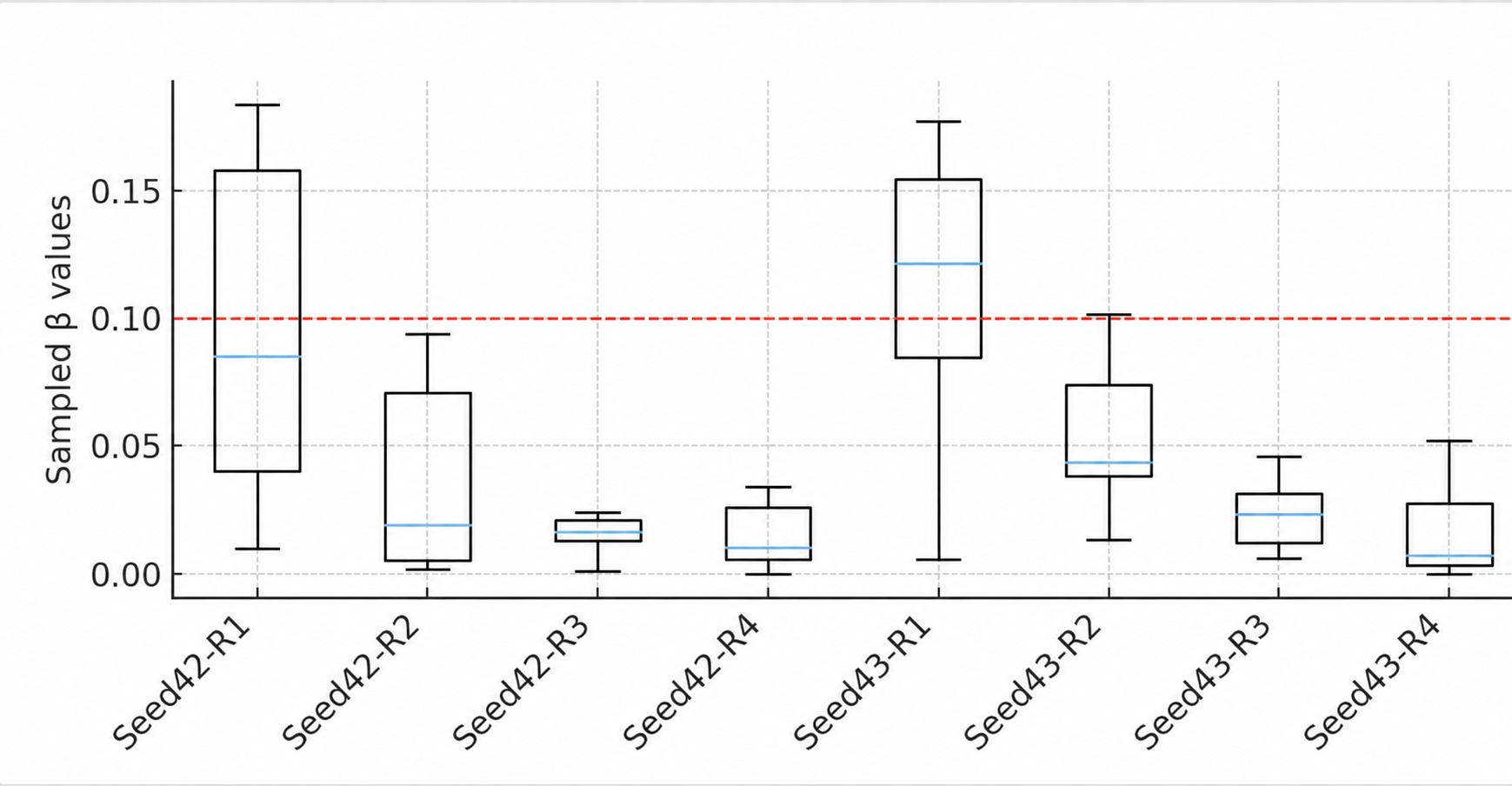}
\caption{Sampled-coefficient distributions for seeds 42 and 43 across four search rounds. The red line marks $\beta=0.1$ as a visual reference, not an empirically established optimum. The plot shows a shift toward smaller sampled values and variation between seeds; it does not establish Bayesian posterior concentration. Sample counts and the precise boxplot convention are not available in the accompanying aggregate records.}
\label{fig:searchsamples}
\end{figure}

\end{document}